\documentclass[11pt]{article}%
\usepackage[hscale=0.7,vscale=0.8]{geometry}
\usepackage{amsmath}
\usepackage{amsfonts}
\usepackage{amssymb}
\usepackage{etex,etoolbox}
\usepackage[ruled, vlined]{algorithm2e}
\usepackage{graphicx}
\makeatletter
\providecommand{\@fourthoffour}[4]{#4}
\def\fixstatement#1{%
  \AtEndEnvironment{#1}{%
    \xdef\pat@label{\expandafter\expandafter\expandafter
      \@fourthoffour\csname#1\endcsname\space\@currentlabel}}}

\globtoksblk\prooftoks{1000}
\newcounter{proofcount}

\long\def\proofatend#1\endproofatend{%
  \edef\next{\noexpand\begin{proof}[Proof of \pat@label]}%
  \toks\numexpr\prooftoks+\value{proofcount}\relax=\expandafter{\next#1\end{proof}}
  \stepcounter{proofcount}}

\def\printproofs{%
  \count@=\z@
  \loop
    \the\toks\numexpr\prooftoks+\count@\relax
\\
\\
    \ifnum\count@<\value{proofcount}%
    \advance\count@\@ne
  \repeat}
\makeatother

\providecommand{\U}[1]{\protect\rule{.1in}{.1in}}
\newtheorem{theorem}{Theorem}
\fixstatement{theorem}

\fixstatement{corollary}

\newtheorem{definition}[theorem]{Definition}

\newtheorem{lemma}[theorem]{Lemma}
\fixstatement{lemma}

\newenvironment{proof}[1][Proof]{\noindent\textbf{#1.} }{\ \rule{0.5em}{0.5em}}
\newenvironment{proof-sketch}[1][Proof sketch]{\noindent\textbf{#1.} }{\ \rule{0.5em}{0.5em}}

\renewcommand{\c}{\mathcal{C}}
\renewcommand{\P}[1]{\mathbb{P}\of{#1}}
\newcommand{\E}[1]{\mathbb{E}\left[#1\right]}
\newcommand{\Esub}[2]{\mathbb{E}_{#1}\left[#2\right]}

\renewcommand{\H}{H}

\newcommand{\f}{\ell}
\renewcommand{\u}{\mathcal{U}}
\renewcommand{\l}{\mathcal{L}}
\newcommand{\of}[1]{\left( #1 \right)}
\renewcommand{\O}[1]{\mathcal{O}\of{#1}}

\newcommand{\rr}{\left\lceil \frac r2 \right\rceil }
\newcommand{\Cov}{\Sigma}

\newcommand{\G}{\mathcal{G}}

\newcommand{\poly}[1]{\text{poly}\of{#1}}

\newcommand{\abs}[1]{\left\vert #1 \right\vert}

\newcommand{\payoff}[2]{\text{payoff}_{#1}\of{#2}}

\newcommand{\reg}{\mathcal{R}}
\newcommand{\tr}{\text{Tr}}
\newcommand{\logdet}[1]{\log \det \of{#1}}

\renewcommand{\b}[1]{\left\{#1\right\}}

\newcommand{\dist}[1]{\Delta\of{#1}}
\newcommand{\R}{\mathbb{R}}
\newcommand{\A}{A}
\renewcommand{\S}{S}
\newcommand{\Sf}{\S_{\f}}
\newcommand{\Slp}{\S_{\A}}
\newcommand{\vp}{\varphi}
\newcommand{\Ssdp}{\S_{\A}}
\newcommand{\Sour}{\S}

\newcommand{\SDP}{\text{SDP}}

\newcommand{\Optf}[1]{\text{OPT}\of{#1}}

\newcommand{\Optsdp}[1]{\text{OPT}_{\SDP}\of{#1}}

\newcommand{\integral}{\l^{\u}}
\newcommand{\rel}{\text{LP}\of{\integral}}

\newcommand{\relpsd}{\text{SDP}\of{\integral}}

\begin{document}

\title{Online Local Learning via Semidefinite Programming}
\author{ Paul Christiano\thanks{UC Berkeley. \ Email: paulfchristiano@eecs.berkeley.edu}}
\date{}
\maketitle

\begin{abstract}

In many online learning problems we are interested in predicting \emph{local} information about some universe of items.  For example, we may want to know whether two items are in the same cluster rather than computing an assignment of items to clusters; we may want to know which of two teams will win a game rather than computing a ranking of teams.  Although finding the optimal clustering or ranking is typically intractable, it may be possible to predict the relationships between items \emph{as well as if} you could solve the global optimization problem exactly.  

Formally, we consider an online learning problem in which a learner repeatedly guesses a pair of labels $(\ell(x), \ell(y))$ and receives an adversarial payoff depending on those labels.  The learner's goal is to receive a payoff as good as the best \emph{fixed} labeling of the items.  We show that a simple algorithm based on semidefinite programming can achieve asymptotically optimal regret in the case where the number of possible labels is constant, resolving an open problem posed by Hazan, Kale, and Shalev-Schwartz \cite{hks}.  Our main technical contribution is a novel use and analysis of the $\log \det$ regularizer, exploiting the observation that $\logdet{\Cov + I}$ upper bounds the entropy of any distribution with covariance matrix $\Cov$.

\end{abstract}

\section{Introduction}

We are often tasked with inferring
the properties of items from observations of their interactions.
Frequently, we are interested in these properties primarily because
they can be used to make predictions about future interactions.
For example, we might:
\begin{itemize}
\item assign documents to clusters in order to make predictions
about their similarity,
\item assign characteristics to users and products in order to
make appropriate recommendations,
\item assign personalities to individuals to predict which groups will function well,
\item assign rankings to teams in order to predict the winners of a sequence of games, or so on.
\end{itemize}
In many of these contexts, we are interested in global assignments
only insofar as they help us make local predictions.
Even when finding a global assignment is intractable,
we may still be able to make predictions
\emph{as well as if} we had found the optimal global assignment.

In Section~\ref{definitions} we present a new formalization of this class of problems.
In the case where each interaction is between two items
and each item has one of $\O{1}$ possible labels,
our formalization is a special case of the matrix learning
framework described in \cite{hks}.
This includes, for example, the online max-cut problem.
In this special case we are able to provide the first asymptotically optimal regret bounds,
reducing the convergence time from $\O{\log(n)}$ per item
to the optimal $\O{1}$,
which may be significant in many applications
(especially those, such as recommendation systems,
where \emph{per-item} regret is a relevant metric).
Moreover, our algorithm and analysis are quite natural,
and considerably simplify previous approaches.

\subsection{Local Prediction Problems}\label{definitions}

\subsubsection{2-Local Prediction}

We fix a universe of items $\u$ and a space of possible labels $\l$,
with $n = \abs{\u}$ and $k = \abs{\l}$.
Let $\dist{\l \times \l}$ be the set of probability distributions over $\l \times \l$.
A \emph{2-local prediction problem} is an online learning problem
where in each round $t = 0, 1, \ldots$

\begin{enumerate}
\item Nature presents a pair $(i_t, j_t) \in \u \times \u$.
\item The learner submits a distribution $p_t \in \dist{\l \times \l}$.
\item Nature reveals a payoff function $c_t : \l \times \l \rightarrow [-1, 1]$.
\item The learner receives 
\[\payoff{t}{p_t} = \sum_{a, b \in \l \times \l}
c_t\of{a, b} p_t\of{a, b}.\]
\end{enumerate}

We are interested in strategies $S : \u \times \u \rightarrow \dist{\l \times \l}$.
For a fixed set of choices by nature,
the payoff of a strategy $S$ can be defined straightforwardly as:
\[ \payoff{T}{S} = \sum_{t = 1}^T \payoff{t}{S\of{i_t, j_t}}. \]
We can expand this definition in the natural way
to consider strategies $S$ which depend on the past
choices of nature.
We will often talk about \emph{efficient} strategies $S$,
which are those strategies for which $S\of{i_t, j_t}$
can be computed from nature's past choices
in time $\poly{t, n, k}$.

We are particularly interested in strategies corresponding to fixed labelings.
For a labeling $\f : \u \rightarrow \l$
we can define the strategy $\Sf$ via:
\[\Sf\of{i, j} = \of{\f(i), \f(j)} \text{with probability 1}.\]
We are interested in finding strategies which perform as well
as the best strategy $\Sf$,
so we define
\[ \Optf{T} = \max_{\f : \u \rightarrow \l} \; \payoff{T}{\Sf}. \]

We prove the following theorem:
\begin{theorem}\label{maintheorem}
There exists an efficient strategy $\Sour$ for $2$-local prediction
that satisfies:
\[ \payoff{T}{\Sour} \geq \Optf{T} - \O{\sqrt{n k^3 T}}. \]
\end{theorem}
The dependence on $n$ and $T$ are optimal,
but the dependence on $k$ is not
(the information-theoretic limit is $\sqrt{n T \log(k)}$).
Fortunately there are already interesting
2-local prediction problems
with $k = \O{1}$,
for which we provide the first asymptotically optimal regret bounds.

\subsubsection{General local prediction}

These definitions can be naturally generalized to $r$-local prediction problems
for $r > 2$.
In this case, nature presents an $r$-tuple and the learner
submits a probability distribution in 
$\Delta\of{\l^r}$.
The $\payoff{t}{p_t}$ are now maps from $\l^r \rightarrow [-1, 1]$, 
and strategies $\Sf$ can be defined exactly analogously
to the case $r = 2$.

\subsection{Examples}

\subsubsection{Online max-cut}

Perhaps the simplest interesting example of a local prediction problem
is online max-cut.
In each round, the learner is given 
a pair of vertices $i_t, j_t \in \u$
and is asked to output a probability distribution
over ``cut'' and ``not cut.''
Nature then picks one of ``cut'' or ``not cut''
and the learner's payoff is the probability they assigned to the correct value
minus the probability they assigned to the incorrect value.

A cut of $\u$ can be understood as a strategy in this game in a natural way.
The goal in the online max-cut problem is to achieve low regret
compared to the best cut.

Our algorithm can be applied to this problem directly:
the space of labels is $\l = \b{+1, -1}$ and
the payoff for a pair $(a, b)$ is simply $1$ if either $a = b$
and nature chose ``not cut'' or $a \neq b$ and nature chose ``cut''
(and $-1$ otherwise).

Note that nothing about the max-cut problem itself 
encodes the fact that we are concerning ourselves with cuts,
except the fact that we are trying to perform as well as the best cut.

\subsubsection{Online gambling}

A significantly more complicated example is the online gambling problem.
In each round, the learner is given a pair of teams $i_t, j_t \in \u$,
and is asked to output a probability distribution over which team will win
in a head-to-head contest.
Nature then picks a winner, and the learner's payoff is the probability they assigned to the 
real winner.
The learner's goal is to achieve low regret compared to the best fixed
ranking of the teams.
(A ranking corresponds to the strategy
of deterministically predicting that the higher-ranked team
wins.)

This problem can be easily fit into our setting by taking $\l = \b{1, 2, \ldots, n}$.
A ranking of the teams then corresponds to assigning each team a separate value.
The payoff for a pair $\f(i_t) = a, \f(j_t) = b$ is $1$ if either $a > b$
and team $i_t$ won or if $b > a$ and team $j_t$ won.

Our algorithm does not perform well on this problem because $k = \omega(1)$.

For now, we view this example as a demonstration
that the local learning framework can accomomdate significant complexity
in natural settings.
In fact it could go much further.
For example, we could assign each team a skill level
and determine a team's probability of victory by the gap between
their skill and their opponents';
we could include additional variables for a team's ability to play well under varying conditions;
and so on.

Thus a solution to the local learning problem for general $k$
would give a very robust solution to the online gambling problem,
in that it could accommodate many natural extensions.
It seems plausible that there is a general solution to the local learning problem,
but in either case it seems to be a natural generalization
of the online gambling problem
and either a positive or negative answer would be of natural interest.

\subsection{Motivation}

A wide range of prediction problems are ``local'' in the sense we have described,
and so understanding the feasibility of such problems is a natural challenge.
Our approach to this problem is motivated by the rich literature
on constraint satisfaction problems in theoretical computer science,
and in particular by the success of semidefinite programming (SDP) techniques.

There is an intuitive connection
between \emph{finding} a single good assignment to some variables
and \emph{competing with} the best fixed assignment.
In practice it seems that both of these problems occur quite frequently:
sometimes we are interested in \emph{finding} a cut of a graph,
and sometimes we are simply interested in \emph{making predictions}
about whether a given pair of items lie on the same or different sides of the cut.

Of course, we would always find a cut \emph{in order to} make predictions.
Our motivating observation was that it might be possible to substantially
improve performance by ``cutting out the middle man'' and using
SDP solutions to make predictions directly.
The key ingredient in this approach is finding
an appropriate regularizer for SDP solutions
which can play the same role as entropy
in traditional inference.
We apply a $\log \det$ regularizer.
By observing that the $\log \det$ of a matrix of moments
is an upper bound on the entropy
of any distribution matching those moments,
we are able to show that the log determinant retains
the convenient properties of entropy regularization.

A final reason to be interested in local learning is that locality serves as a building block
for many other natural structural assumptions.
For example, circuits and graphical models are both generated
by a set of local data.
Understanding the interaction between locality and learnability
is a natural step in probing the boundaries of learnability,
and finding simple approaches to the local learning problem
may help extend positive results to more difficult cases.

\subsection{Relaxations}

We will work extensively with a number of relaxations of $\dist{\integral}$.
In particular, we say that a matrix $\A \in \R^{nk \times nk}$ 
with rows and columns indexed by pairs in $\u \times \l$
is a \emph{pseudodistribution over labelings} if
$a, b \mapsto \A_{(i, a)(j, b)}$ is a valid distribution in $\dist{\l^2}$ for each $i, j$:
that is, if $A_{(i, a)(j, b)} \geq 0$ for all $i, j, a, b$ and $\sum_{a, b} A_{(i, a)(j, b)} = 1$ for all $i, j$\footnote{We
might additionally require that the marginals are consistent, i.e. that for all $a, i, j, k$,
$\sum_b A_{(i, a)(j, b)} = \sum_b A_{(i, a)(k, b)}$. But our analysis won't make any use of that condition, and so we omit it.
}.


Intuitively, $A_{(i, a)(j, b)}$ represents $\P{\f(i) = a \wedge \f(j) = b}$
for a putative distribution $\mathbb{P}$,
though in fact there might not be any underlying distribution
that reproduces these probabilities.
Write $\rel$ for the space of pseudodistributions.
In order to tighten the relaxation we might additionally 
require that $\A$ be positive semidefinite;
write $\relpsd$ for the space of positive semidefinite pseudodistributions.
It is easy to verify that any probability distribution $\dist{\integral}$
corresponds to a matrix in $\relpsd$.

For any pseudodistribution $\A$,
we can define a strategy $\Slp$ by 
\[ \Slp\of{i, j} = \of{a, b \mapsto \A_{(i, a)(j, b)}}. \]

We define
\[ \Optsdp{T} = \max_{A \in \relpsd} \payoff{T}{\Ssdp},\]
and our main theorem will establish:
{
\renewcommand{\thetheorem}{\ref{maintheorem}}
\begin{theorem}
There is an efficient  strategy $\Sour$ for 2-local prediction such that
\[ \payoff{T}{\Sour} \geq \Optsdp{T} - \O{\sqrt{n k^3 T}}, \]
and in particular
\[ \payoff{T}{\Sour} \geq \Optf{T} - \O{\sqrt{n k^3 T}}, \]
\end{theorem}
\addtocounter{theorem}{-1}
}

It is easy to see that the positive-semidefiniteness constraints
are necessary:
without those constraints, $\rel$ just treats each of the $n^2$
pairs $(i, j)$ as a separate problem, and so it is impossible
to achieve regret $o\of{\sqrt{n^2 T}}$..
Moreover, requiring consistency of the marginals does not help in general.
For example, in a max-cut problem we can assume by symmetry that the
marginal distribution of each item's label is uniform.

%

\subsection{Related work}

Some of the earliest work on online learning considered the problem
of competing with the best strategy from an unstructured set;
for example, see \cite{Blum98, Cover96, FV97, FS97}.
This work achieved optimal regret bounds, but the approach is not directly applicable
in settings where we would like to compete with an implicitly defined
and exponentially large class of strategies.
These protocols can be understood as instances of the general
frameworks of follow the regularized leader (FTRL) or mirror descent\cite{ftrl},
as can our algorithm.

Similar techniques have now been applied to a much broader range of problems,
including many settings with large, implicit classes of strategies \cite{BCK03, Freund97, HS97, KV05, TK02}.
In most cases, these protocols are applied to settings where the corresponding
offline decision problem is easy.
Our work continues in this vein, aiming to compete
with a large class of combinatorially defined strategies,
but we are interested in the setting where optimization
over the space of strategies is intractable.

Another line of work considers online learning against spaces of positive
semidefinite matrices.
The matrix multiplicative weights algorithm competes with the class
of positive definite matrices by using an implicit von Neumann entropy regularizer
\cite{MWM}.
In a similar vein,
\cite{online-burg, implicit-learning, seraph} compete with the class of metrics.
Although the problem statement is quite different,
we apply similar techniques.

The recent results of Hazan, Kale, and Shalev-Schwartz \cite{hks} in particular are closely related to our own.
They consider the setting in which learners produce outputs in $[-1, 1]$
and compete with the class of matrices which can be decomposed as a difference
of positive semidefinite matrices with small entries and small trace.
Their framework is equivalent to ours in the case of 2-local prediction for $k = \O{1}$.
In that setting
they obtain a regret bound of $\O{\sqrt{n \log\of{n} T}}$, which differs from our bound
by a $\log\of{n}$ factor.
This is the difference between a \emph{per item} convergence time of $\O{1}$
and $\O{\log(n)}$, which may be quite significant in some settings.
For example, this might be the difference between a single user
needing to wait $\O{\log(n)}$ time before receiving good recommendations
and needing to wait $\O{1}$ time.

Our techniques differ from those of \cite{hks} primarily by our choice of regularizer.
Like us, they work with a semidefinite relaxation for the space of labelings
(though they do not describe their approach in these terms)
and solve an appropriately regularized problem.
Their regularizer is the von Neumann entropy,
and this leads to their regret bound of $\O{\sqrt{n \log(n) T}}$.
We are able to improve this bound by using the log determinant.
Moreover, because the log determinant is a natural analog of entropy
in the setting of semidefinite programming relaxations for constraint satisfaction, 
we are able to give a conceptually simple analysis.

The log determinant regularizer is new in this setting
but has been applied before in online learning,
particularly in the context of metric learning \cite{online-burg, implicit-learning, seraph}.
Our use of this regularizer appears to be based on a fundamentally different motivation,
namely as an estimate of entropy given a matrix of moments.
The log determinant regularizer has also appeared (with a more similar motivation)
as a regularizer for relaxed inference over Markov random fields \cite{logdet-bound},
though in their context the analysis is quite different.
This regularizer is perhaps most common
as a barrier function in semidefinite programming;
there is an analogy between our approach
and a path-following algorithm for semidefinite programming,
though again our analysis is quite different.


\section{Our algorithm}\label{algorithm}

In this section we will exhibit an algorithm which achieves regret $\O{\sqrt{n k^3 T}}$ against
the class of all strategies $S_A$ for $A \in \relpsd$.
Because $S_A$ includes all strategies $S_{\ell}$ for $\ell : \u \rightarrow \l$,
this yields the desired result.

\subsection{Follow-the-Regularized Leader}
A common approach to achieving an online regret bound
is \emph{follow-the-regularized-leader} (FTRL).
FTRL chooses each move using a strategy
which maximizes a linear combination of
\emph{retrospective performance} and a \emph{regularization term}.
After nature reveals each payoff,
FTRL recomputes the optimum of this objective function
and uses it to make a decision in the next round.

If the regularization term is strongly concave,
then we can show that the maximizing strategy does not change much from one round
to the next.
If in addition the regularization term is bounded,
we can obtain a bound on the total regret.

For completeness, we will briefly describe the algorithm,
following the presentation in \cite{ftrl}.
We also give a precise statement of the bounds we need to prove.

Formally, let $\c$ be a convex set and
let $\reg : \c \rightarrow \R$ be a strongly concave function which will serve
as a regularizer.
Consider the algorithm:

\begin{algorithm}
\SetKwInOut{Input}{input}
\Input{A strongly concave regularizer $\reg$,
an $\eta > 0$,
and $\c$ a convex subset of $\rel$}
\BlankLine
\For{$t \in \mathbb{N}$}{
Set 
\[A = \arg \max_{A \in \c}{ \eta \sum_{s = 1}^{t-1} \payoff{s}{A} + \reg\of{A}}\]
Play $(a, b)$ with probability $A_{(i_t, a)(j_t, b)}$.
}
\caption{Follow the Regularized Leader\label{ftrl-algorithm}}
\end{algorithm}

\begin{theorem}\cite{ftrl}
Suppose that $\reg$ is bounded between $0$ and $M$ on $\c$.
Furthermore, suppose that whenever $\abs{\payoff{S}{A} - \payoff{S}{A'}} \geq \delta$,
\[ \reg\of{\epsilon A + (1 - \epsilon) A'} \geq 
\epsilon \reg\of{A} + (1 - \epsilon) \reg\of{A'} + 
\epsilon (1 - \epsilon) \gamma \delta^2\]
Then for an appropriate choice of $\eta$ the performance of algorithm~\ref{ftrl-algorithm}
is within an additive $\O{\sqrt{ \frac{T M}{\gamma} }}$ of
$\max_{A \in \c} \payoff{T}{A}$.
\end{theorem}

We will aim to apply this theorem with $M = n k$ and $\gamma = \frac 1{k^2}$,
leading to an algorithm with regret $\sqrt{ n k^3 T}$, as desired.
We believe that this analysis can be tigthened to give a regret bound of $\sqrt{n k T}$
by generalizing Lemma~\ref{variation-distance} and consequently Lemma~\ref{logdet-convexity},
but we do not have a proof (and in either case our algorithm is most likely to be of interest when $k = \O{1}$).

\subsection{Regularizing relaxations.}
In order to apply FTRL we need to 
choose a convex set $\c$
and find a regularizer $\reg$ which has the desired properties.

One \emph{intractable} approach would be to take $\c$
to be the space of distributions over labelings,
and to take $\reg$ to be the entropy $\H$. 
\begin{definition}[Entropy]
For a discrete random variable $X$ with distribution $p(x)$,
the discrete entropy is $\H\of{X} = \E{-\log\of{p(X)}}$.
For a continuous random variable $X$ over $\mathbb{R}^N$ with density $p(x)$,
the differential entropy of $X$ is $\H\of{X} = \E{-\log\of{p(X)}}$.
\end{definition}

Since there are only $k^n$ possible labelings
the entropy is bounded by $n \log(k)$.
Moreover, the following standard lemma shows that the entropy is strongly concave
(the proof is in the appendix):
\begin{lemma}\label{entropy-convexity}
If $P$ and $Q$ are probability distributions with total variation distance $\delta$,
then 
\[\H\of{\epsilon P + (1 - \epsilon) Q} 
\geq \epsilon\H\of{P} + (1 - \epsilon) \H\of{Q} + \epsilon (1 - \epsilon) \delta^2, \]
where $\H$ is either the discrete entropy or the differential entropy.
\end{lemma}
\proofatend
Because the differential entropy can be written as a limit
of discrete entropies under partitions,
it is sufficient to verify the claim for discrete entropy.

Consider an enumeration of the union
of the supports of $P$ and $Q$, $x_1, x_2, \ldots$.
Let $p(x_i)$ be the probability that $P$ assigns to $x_i$
and let $q(x_i)$ be the probability that $Q$ assigns to $x_i$.
For $z \in \R$, let $\H\of{z} = z \log(z)$.
Then
\begin{align*}
\H\of{\epsilon P + (1 - \epsilon) Q}
&= \sum{ \H\of{\epsilon p(x_i) + (1 - \epsilon) q(x_i)}} \\
&\geq \sum{ \epsilon \H\of{p(x_i)} + (1 - \epsilon) \H\of{q(x_i)} 
+ \frac 12 \epsilon(1 - \epsilon) \frac {\of{p(x_i) - q(x_i)}^2}{p(x_i) + q(x_i)}} \\
&\geq \epsilon \H\of{P} + (1 - \epsilon)\H\of{Q} + \frac 14 \epsilon(1 - \epsilon) \of{\sum{\abs{p(x_i) - q(x_i)}}}^2 \\
&= \epsilon \H\of{P} + (1 - \epsilon) \H\of{Q} + \frac 14 \epsilon (1 - \epsilon) \delta^2
\end{align*}
where the first inequality holds because $\H''(z) = \frac 1z$
and the second inequality is by Cauchy-Schwarz and the fact that $\sum\of{p(x_i) + q(x_i)} = 2$.
\endproofatend

Of course, even representing a probability distribution over assignments $\u \rightarrow \l$
is intractable.
Moreover, for our application it is unnecessary:
because our payoffs depend on the labels of pairs of items,
we only need to represent the pairwise marginal distributions of our labeling,
i.e. the probabilities $\P{\f\of{i} = a \wedge \f\of{j} = b}$.
This is precisely the information encoded by some $\A \in \rel$.
This motivates us to instead take the set $\c = \rel$,
and search for some regularizer that can play the same role as $\H$.

Unfortunately, the space $\rel$ is too large,
and regardless of our choice of regularizer
there are information-theoretic obstructions to obtaining
a sub-quadratic regret bound.

However, if we take the space $\relpsd$, the situation changes entirely.
Although there still need not be an actual distribution which satisfies
\[\P{\f\of{i} = a \wedge \f\of{j} = b} = \A_{(i, a)(j, b)},\]
we now \emph{can} find a jointly normally distributed family of random variables $X_{(i, a)}$
such that 
\[ \E{X_{(i, a)} X_{(j, b)}} = A_{(i, a)(j, b)}.\]
Because the gaussian is the maximum entropy distribution
subject to these moment constraints,
its differential entropy provides a natural upper bound on the entropy
of a putative probability distribution which does have moments
given by $A$.
This motivates using the entropy of this gaussian as our regularizer $\reg$.

The differential entropy of a gaussian with covariance matrix $\Cov$
is $\O{\logdet{\Cov}}$.
This motivates us to explore the suitability of this function as a regularizer.
In fact we smooth this function to $\logdet{\Cov + \frac 1k I}$
to account for the difference between a discrete entropy and a differential entropy
(as in \cite{logdet-bound}).
For convenience of normalization, we actually take $\reg\of{\A} = \logdet{k A + I}$.

\subsection{The log-determinant regularizer}

In this section write $\c = \relpsd$ and $\reg\of{\A} = \logdet{k \A+ I}$.
It is well known that maximization of concave functions over $\relpsd$
is tractable,
so it remains to show that this choice of $\reg$ is
bounded and strongly concave.

\begin{lemma}
For any $\A \in \relpsd$, $0 \leq \reg\of{\A} \leq nk$.
\end{lemma}
\begin{proof}
Since $\A \geq 0$, each eigenvalue of $k \A+I$ is at least $1$
and hence $\logdet{\A + \frac 1k I} \geq 0$.
On the other hand, $\tr\of{\A} = n$,
so $\tr\of{\A + \frac 1k I} = 2kn$.
Subject to this trace condition, $\det\of{\A+\frac 1k I}$ is maximized
if all eigenvalues are equal, i.e. $k \A + I = 2 I$.
Hence $\logdet{\A + \frac 1k I} \leq \logdet{2 I} = n k$.
\end{proof}

The strong concavity of $\logdet{k \A+I}$ is equivalent
to the strong concavity of $\logdet{\Cov}$.
In order to verify strong concavity of $\logdet{\Cov}$,
we utilize its characterization
as the entropy of a gaussian with covariance matrix $\Cov$
and apply the strong concavity of entropy.
First, we need the following standard lemma
(for example, see \cite{ct-eit-91}):

\begin{lemma}\label{gaussian-entropy}
For any distribution $X$ over $\R^{nk}$ with covariance matrix $\Cov$,
the differential entropy of $X$ is at most 
$\frac 12 \logdet{\Cov} + H_0(nk)$,
where $H_0$ is independent of $X$ and $\Cov$.
Moreover, equality is attained if $X$ is gaussian.
\end{lemma}

We also need to show that gaussians which differ in one moment
necessarily have large variation distance (proved in the appendix):
\begin{lemma}\label{variation-distance}
Let $\G_1$ and $\G_2$ be gaussians
with covariance matrices $\Cov_1$ and $\Cov_2$.
Suppose that for some $i, j$ we have
\[\abs{(\Cov_1)_{ij} - (\Cov_2)_{ij}} \geq \delta \of{ \of{\Cov_1}_{ii} + \of{\Cov_1}_{jj} + \of{\Cov_2}_{ii} + \of{\Cov_2}_{jj}  }. \]
Then the total variation distance between $\G_1$
and $\G_2$ is $\Omega\of{\delta}$.
\end{lemma}
\proofatend
The following proof was suggested by George Lowther in response to a question by the author on \texttt{mathoverflow.net}.

Without loss of generality, assume $\of{ \of{\Cov_1}_{ii} + \of{\Cov_1}_{jj} + \of{\Cov_2}_{ii} + \of{\Cov_2}_{jj}  } = 1$,
so that $\Cov_1$ and $\Cov_2$ differ by at least $\delta$ in their $i, j$ entry.
Consider the characteristic function $\vp_k(u) = \Esub{x \sim \G_k}{ \exp\of{i u^T  x}}$.
For $u$ with real entries, $\vp_k(u)$ is the expectation of a function with absolute value $1$
on a samples drawn from $\G_k$,
so to bound the variation distance between $\G_1$ and $\G_2$
it suffices to exhibit some $u$ such that $\abs{\vp_1(u) - \vp_2(u)} = \Omega\of{\delta}$.
From the definition of the Gaussian we can compute
\[ \vp_k(u) = \exp\of{- \frac 12 u^T \Cov_k u}. \]

Let $e_i, e_j$ be the unit vectors with their non-zero entry in coordinates $i$, $j$, respectively.
Let $v$ be one of $e_i, e_j$, or $e_i + e_j$,
write $\alpha_k = v^T \Cov_k v$, and write $u = \frac {v}{\sqrt{\alpha_1 + \alpha_2}}$.
Then we have
\begin{align*}
\abs{\vp_1(u) - \vp_2(u)} 
&= \abs{ \exp\of{- \frac {\alpha_1}{ 2 (\alpha_1 + \alpha_2)}}  - \exp\of{- \frac {\alpha_2}{2 (\alpha_1 + \alpha_2)}}} \\
&= \Omega\of{\frac {\abs{\alpha_1 - \alpha_2}}{\alpha_1 + \alpha_2}} \\
&= \Omega\of{\abs{\alpha_1 - \alpha_2}}
\end{align*}
Where the second line follows from a constant lower bound on the derivative of $\exp\of{x}$ in the range $[-1, 0]$,
and the third line follows from our bound on 
$\of{ \of{\Cov_1}_{ii} + \of{\Cov_1}_{jj} + \of{\Cov_2}_{ii} + \of{\Cov_2}_{jj}  }$.

So it suffices to find some $v \in \left\{e_i, e_j, e_i + e_j\right\}$ with $\abs{v^T \of{\Cov_1 - \Cov_2} v} = \Omega\of{\delta}$.
But note that we can write
\begin{align*}
\delta &\leq \of{\Cov_1 - \Cov_2}_{ij}  \\
&= e_i^T \of{\Cov_1 - \Cov_2} e_j \\
&= \frac 12 \of{(e_i + e_j)^T \of{\Cov_1 - \Cov_2} (e_i + e_j) - e_i^T \of{\Cov_1 - \Cov_2} e_i - e_j^T \of{\Cov_1 - \Cov_2} e_j }
\end{align*}
and so one of the terms on the right hand side of the second line must have absolute value at least $\Omega\of{\delta}$,
as desired.

\endproofatend
Now we can prove that $\logdet{\Cov}$ is strongly concave:
\begin{lemma} \label{logdet-convexity}
Suppose that $\Cov_1$ and $\Cov_2$ are as in Lemma~\ref{variation-distance}.
Then
\[\logdet{(1 - \epsilon) \Cov_1 + \epsilon \Cov_2} 
\geq (1 - \epsilon) \logdet{\Cov_1} + \epsilon \logdet{\Cov_2} + 
\Omega\of{\epsilon(1-\epsilon)\delta^2}.\]
\end{lemma}
\begin{proof}
Let $\G_i$ be the gaussian with covariance matrix $\Cov_i$.
Let $M$ be the probabilistic mixture of gaussians
which puts probability $(1 - \epsilon)$ on $\G_1$
and $\epsilon$ on $\G_2$.
By Lemma~\ref{gaussian-entropy}, $H_0(nk) + \logdet{(1 - \epsilon) \Cov_1 + \epsilon \Cov_2}$
is an upper bound on the entropy of any distribution with covariance matrix
$(1 - \epsilon) \Cov_1 + \epsilon \Cov_2$, and in particular on $M$.
Moreover, we can lower bound
the entropy of $M$ by Lemma~\ref{entropy-convexity} and our lower bound
on the variation distance between $\G_1$ and $\G_2$:
\begin{align*}
\logdet{(1 - \epsilon) \Cov_1 + \epsilon \Cov_2}
&\geq
\H\of{M} - H_0(nk) \\
&\geq (1 - \epsilon) \H\of{\G_1} + \epsilon \H\of{\G_2} + 
\Omega\of{\epsilon (1 - \epsilon)  \delta^2}  - H_0(nk)\\
&\geq (1 - \epsilon) \of{\H\of{\G_1} - H_0(nk)} + \epsilon \of{\H\of{\G_2} - H_0(nk)} + 
\Omega\of{\epsilon (1 - \epsilon)  \delta^2} \\
&=
(1 - \epsilon) \logdet{\Cov_1} + \epsilon \logdet{\Cov_2}
+ \Omega\of{\epsilon(1 - \epsilon) \delta^2}
\end{align*}
as desired.
\end{proof}
\begin{lemma}
If $\payoff{t}{A} - \payoff{t}{A'} \geq \delta$,
then \[\reg\of{\epsilon A + (1 - \epsilon) A'}
\geq \epsilon \reg\of{A} + (1 - \epsilon) \reg\of{A'}
+ \Omega\of{\epsilon (1 - \epsilon)\frac {\delta^2}{k^2}} \]
\end{lemma}
\begin{proof-sketch}
If $\payoff{t}{A} - \payoff{t}{A'} \geq \delta$,
then $kA$ and $kA'$ differ by at least $\frac {\delta}{k}$
in the average entry in the block corresponding to the pair $(i_t, j_t)$.
The average diagonal entry of $kA + I$ or $kA' + I$ is only $2$.
So there is necessarily some entry in the block corresponding to the pair $(i_t, j_t)$ for which the conditions of Lemma~\ref{logdet-convexity} apply,
with parameter $\frac {\delta}{8k}$.
\end{proof-sketch}
\proofatend
Write $d_{ab} = k \abs{A_{(i_t, a)(j_t, b)} - A'_{(i_t, a)(j_t, b)}}$,
$x_a = k A_{(i_t, a)(i_t, a)} + k A'_{(i_t, a)(i_t, a)} + 2$,
$y_b = k A_{(j_t, b)(j_t, b)} + k A'_{(j_t, b)(j_t, b)} + 2$.
Note that $\sum_a x_a = \sum_b y_b = 4k$.
Note that
\[ 
\sum_{a, b} d_{ab} = k \abs{\payoff{t}{A} - \payoff{t}{A'}} \geq k \delta.
\]
Since $\sum_{a,b} (x_a + y_b) = k \sum x_a + k \sum y_b = 8k^2$,
there must be some $a, b$ with $d_{ab} \geq \frac 1{8k} (x_a + y_b)$.
The desired result then follows from Lemma~\ref{logdet-convexity}.
\endproofatend

Together with the analysis of FTRL, this completes our algorithm.

\section{Conclusion}

\subsection{Further work}

\subsubsection{$k = \omega(1)$}

Our algorithm is optimal in the range $k = \O{1}$,
but has very bad performance for large $k$;
similarly,
the algorithm of \cite{hks} does not generally apply to large $k$.
Solving the problem for large $k$
would give a robust solution to online gambling 
problem which would also apply directly to many natural generalizations
and related learning problems.

It seems quite likely that follow the regularized leader 
can achieve regret $\sqrt{n \log(k) T}$ using
the relaxation $\relpsd$,
given an appropriate choice of regularizer.

In the case of $n = 1$,
this problem reduces to conventional learning from experts,
for which an entropy regularization
is suitable.
Intuitively, what is needed is a way to integrate
this entropy regularization (or the von Neumann entropy regularization of
matrix multiplicative weights) with the $\log \det$ regularization 
that performs well for large $n$.
One way of understanding this problem
is as a search for a notion of entropy
that applies to arbitrary matrices $\A \in \relpsd$.
The bound $\logdet{A+I}$ treats each indicator variable $\f(i) = a$
separately, and so ignores the fact that for each $i$
the events $\f(i) = a$ are mutually exclusive.
This causes it to be a conservative overestimate for the entropy,
and so to yield a suboptimal regret bound.

\subsubsection{$r > 2$}

In contrast with the $k = \omega(1)$ case,
extending these results to $r > 2$ seems likely to be extremely difficult.
It may be possible to apply existing semidefinite programming hierarchies,
but if so it requires maintaining more than the first $r$ moments of a pseudodistribution,
and would be a new and interesting form of evidence for the usefulness of higher
levels of these hierarchies.
If existing hierarchies cannot solve the problem,
it seems likely that solving $r$-local prediction problems
for $r > 2$ would require a significant conceptual developments.
It is also possible that there are complexity-theoretic obstructions,
but proving lower bounds under conventional assumptions seems to be out of reach for usual techniques.

Note that we can immediately apply our results
to obtain regret $\O{\sqrt{n^{\rr} T}}$ in the case $k = \O{1}$,
by assigning labels in $\l^{\rr}$ to each set of $\left \lceil \frac r2 \right \rceil $ items
from $\u$.
Understanding whether any efficient algorithm can do better
may shed light on learning more broadly but also on constraint satisfaction.

\subsubsection{Richer structure}

Local prediction problems have particularly clean structure
which makes them amenable to semidefinite programming (at least in simple cases).
In many applications of interest,
predictions have slightly richer structure.
For example, a prediction might depend on the label of item $i$
together with the label of one additional item 
which is chosen adaptively based on the label of $i$.
This would occur for example if each of $n$ items belonged to one of $k$ clusters,
and the properties of the items were stochastic functions
of the (unobserved) characteristics of the clusters that contained them.
It is interesting to ask how far we can go in this direction before impossibility
results kick in; even relatively modest progress might allow completely automatic
inference in a wide range of natural models.

\subsection{Discussion}

We have introduced the family of \emph{local prediction problems,}
and left most questions concerning this model open.
We have shown that the simplest local prediction problems
can be solved essentially optimally by a particularly natural algorithm,
providing the first asymptotically optimal regret bounds
for the online max-cut problem.
In our view the conceptual simplicity of our analysis is a significant strength;
we take its simplicity
(together with its improvement over the previous state of the art)
as some evidence that our view of online local learning
is a productive one, and that further developments may be possible.

\bibliographystyle{plain}

\bibliography{thesis}

\section{Appendix}

\printproofs

\end{document}